\pdfoutput=1

\documentclass[11pt]{article}

\usepackage{acl}
\usepackage{times}
\usepackage{latexsym}

\usepackage[T1]{fontenc}

\usepackage[utf8]{inputenc}

\usepackage{microtype}

\usepackage{inconsolata}

\usepackage{graphicx}

\usepackage{amsmath}
\usepackage{amssymb}
\usepackage{bm}

\usepackage{amsfonts}
\usepackage{amssymb}

\usepackage{algorithm}

\usepackage{algpseudocode}
\usepackage{times}
\usepackage{latexsym}
\usepackage{soul}
\usepackage[T1]{fontenc}
\usepackage{verbatim}
\usepackage{booktabs}
\usepackage{hyperref}
\usepackage{multirow}
\usepackage[table]{colortbl}
\usepackage{graphicx}
\usepackage{amsmath}
\usepackage{caption}
\usepackage{subcaption}
\usepackage{tabularx}
\usepackage{url}
\usepackage{xurl}
\usepackage[most]{tcolorbox}
\usepackage{amsthm}

\theoremstyle{plain}
\newtheorem{theorem}{Theorem}[section]
\newtheorem{lemma}[theorem]{Lemma}

\theoremstyle{definition}

\theoremstyle{remark}

\title{The Confidence Paradox: Can LLM Know When It’s Wrong?}

\author{
\textbf{Sahil Tripathi}\textsuperscript{1},
\textbf{Md Tabrez Nafis}\textsuperscript{1}\footnotemark[1],
\textbf{Imran Hussain}\textsuperscript{1},
\textbf{Jiechao Gao}\textsuperscript{2}\thanks{Corresponding Author:  jiechao@stanford.edu, tabrez.nafis@gmail.com} \\
\textsuperscript{1}Jamia Hamdard, New Delhi, India \\
\textsuperscript{2}Center for SDGC, Stanford University, California, USA \\
}

\begin{document}
\maketitle
\begin{abstract}
Document Visual Question Answering (DocVQA) models often produce overconfident or ethically misaligned responses, especially under uncertainty. Existing models like LayoutLMv3, UDOP, and DONUT focus on accuracy but lack ethical calibration. We propose \textbf{HonestVQA}, a model-agnostic, self-supervised framework that aligns model confidence with correctness using weighted loss and contrastive learning. We introduce two new metrics—Honesty\footnote{In our work, \textit{honesty} is operationalized as the alignment between model confidence and correctness, with the goal of reducing confidently wrong answers.} Score (H-Score) and Ethical Confidence Index (ECI)—to evaluate ethical alignment. \textbf{HonestVQA} improves accuracy and F1 by up to 4.3\% across SpDocVQA, InfographicsVQA, and SROIE datasets, while reducing overconfidence. It also generalizes well across domains, achieving 78.9\% accuracy and 76.1\% F1-score.  
\end{abstract}

\section{Introduction}
\label{Introduction}

Document Visual Question Answering (DocVQA) has emerged as a key challenge in multimodal AI \cite{wang2025docvideoqa}, enabling models to answer questions based on visual and textual content in documents such as invoices, forms, contracts, and academic papers. These models are widely deployed in enterprise automation \cite{jiang2024enhancing}, legal analysis \cite{liu2024ledqa}, and assistive technologies \cite{zeng2024advancing}. However, despite their growing utility, DocVQA models often lack ethical transparency—frequently returning confidently incorrect answers to ambiguous, adversarial, or under-specified queries. For instance, a model may assert the total invoice amount with high confidence even when the relevant table is partially occluded, or confidently misinterpret a scanned signature line as a date. Such failures can propagate serious downstream consequences, including legal misinterpretation, misinformation, or financial misjudgment.

However, the crux of the problem lies in the inability of existing DocVQA models to communicate uncertainty in a calibrated, ethically responsible manner. While State-of-the-Art (SOTA) models such as LayoutLMv3\footnote{\url{https://huggingface.co/microsoft/layoutlmv3-base}} \cite{fujitake2024layoutllm}, UDOP\footnote{\url{https://huggingface.co/microsoft/udop-large}} \cite{wang2023docllm}, and DONUT\footnote{\url{https://huggingface.co/naver-clova-ix/donut-base}} \cite{li2024enhancing} focus on improving accuracy through sophisticated architecture and pretraining strategies, they fall short in aligning model confidence with actual knowledge. LayoutLMv3 \cite{fujitake2024layoutllm}  tends to prioritize exact answers over conveying doubt, UDOP \cite{wang2023docllm} frequently errs on the side of over-caution without actionable explanations, and DONUT \cite{li2024enhancing} offers no uncertainty estimation at all—leading to ethically untrustworthy behavior in ambiguous scenarios. Therefore, recent advances in AI alignment research have emphasized the importance of ethical calibration \cite{rao2023ethical}, including honesty \cite{yang2024alignment}, confidence-awareness \cite{stangel2025rewarding}, and transparent failure modes \cite{stewart2023large}. However, these insights have yet to be meaningfully integrated into DocVQA models.

To address these critical gaps, we propose \textbf{HonestVQA}, a self-supervised framework that calibrates model confidence to reflect its underlying knowledge and ethical responsibility. Our approach is model-agnostic and integrates three key components: (1) uncertainty quantification to identify knowledge gaps, (2) confidence-accuracy alignment through weighted loss optimization, and (3) contrastive learning to enforce ethical response boundaries in ambiguous contexts. We also introduce two novel evaluation metrics: i) Honesty Score (H-Score), which captures the alignment between confidence and correctness, and ii) Ethical Confidence Index (ECI), which evaluates whether high-confidence answers are ethically warranted. 

\section{Related Work}

Recent research has increasingly focused on improving the reliability and interpretability of AI models, especially in high-stakes domains. While core DocVQA models like have been discussed in Section \ref{Introduction}, here we focus on complementary areas that our framework draws from—confidence calibration, ethical modeling, and contrastive learning. Confidence calibration models such as temperature scaling \cite{xie2024calibrating} and label smoothing \cite{muller2019does} aim to align predicted probabilities with empirical accuracies. However, these models are typically post-hoc and task-agnostic, often failing to generalize in multimodal settings. Selective prediction frameworks such as \cite{chen2023adaptation} allow models to abstain from uncertain answers, but they usually rely on fixed thresholds and lack principled mechanisms to model epistemic uncertainty in visually grounded tasks like DocVQA. However, in the area of ethical and honest AI, efforts such as instruction tuning for alignment \cite{zhang2023instruction} and calibrated language modeling \cite{zhu2023calibration} emphasize epistemic humility—training models to express uncertainty when appropriate. However, these researches are primarily developed for language-only models and remain underexplored in multimodal tasks involving structured visual data. Whereas, contrastive learning has shown strong performance in aligning multimodal representations, with frameworks like CLIP \cite{gao2024clip} and ALIGN \cite{wang2023making} leveraging contrastive objectives for image-text alignment. While effective at learning generalizable embeddings, such models are not designed to enforce ethical boundaries or distinguish between honest and overconfident outputs in ambiguous scenarios.

\begin{algorithm}[ht]
\caption{HonestVQA Training Algorithm}
\label{alg:honestvqa}
\begin{algorithmic}[1]
\Require Pretrained model $f_\theta$, input $(D, Q, y^*)$, thresholds $\delta$, $\tau_1$, $\tau_2$, weights $\alpha$, $\beta$, $m$, $\lambda_1$, $\lambda_2$
\Ensure Calibrated DocVQA wrapper

\State Compute $P(y \mid D, Q) \gets f_\theta(D, Q)$
\State Compute confidence $\mathcal{C} = \max_i P(y_i)$ and entropy $\mathcal{U} = -\sum_i P(y_i) \log P(y_i)$
\State Predict $\hat{y} \gets \arg\max_y P(y)$

\State $\mathcal{L}_{\text{align}} \gets \alpha \cdot \mathbb{1}[\hat{y} \neq y^*] \cdot \mathcal{C} + \beta \cdot \text{CE}(\hat{y}, y^*)$

\If{$\text{WMD}(\hat{y}, y^*) < \delta$}
    \State $h_{\text{pos}} \gets \text{Embed}(\hat{y})$
\EndIf
\If{$\hat{y} \neq y^* \land \mathcal{C} > \tau_1 \land \mathcal{U} < \tau_2$}
    \State $h_{\text{neg}} \gets \text{Embed}(\hat{y})$
\EndIf

\State Compute $\mathcal{L}_{\text{contrast}} = \max(0, m - \text{sim}(h_{\text{anchor}}, h_{\text{pos}}) + \text{sim}(h_{\text{anchor}}, h_{\text{neg}}))$

\State $\mathcal{L}_{\text{total}} \gets \lambda_1 \cdot \mathcal{L}_{\text{align}} + \lambda_2 \cdot \mathcal{L}_{\text{contrast}}$

\State Update projection head using $\mathcal{L}_{\text{total}}$
\end{algorithmic}
\end{algorithm}

\section{Methodology}

As discussed earlier, \textbf{HonestVQA} is a model-agnostic calibration framework designed to enhance ethical transparency in DocVQA models. It operates as a wrapper around pretrained DocVQA models (in our work, we evaluate our framework on top of pretrained models such as LayoutLMv3 \cite{fujitake2024layoutllm}, UDOP \cite{wang2023docllm}, and DONUT \cite{li2024enhancing} to demonstrate its generalizability), injecting uncertainty-aware alignment and contrastive reasoning to reduce overconfident yet incorrect outputs. The broader process of the \textbf{HonestVQA} is illustrated in Algorithm \ref{alg:honestvqa}.

\subsection{Uncertainty Quantification Module}

Given a document $D$ and a question $Q$, we use a pretrained DocVQA model $f_\theta$ that maps $(D, Q)$ to an answer distribution $P(y \mid D, Q; \theta)$. To quantify the model's epistemic uncertainty, we compute the softmax entropy of the output distribution according to Equation (1).
\[\small
\mathcal{U}(D, Q) = -\sum_{i=1}^{|Y|} P(y_i \mid D, Q) \log P(y_i \mid D, Q) \tag{1}
\]
Here, $|Y|$ denotes the size of the answer space. Higher entropy corresponds to greater uncertainty. We also define a maximum-confidence score as shown in Equation (2).
\[\small
\mathcal{C}(D, Q) = \max_i P(y_i \mid D, Q)\tag{2}
\]
This dual view captures both dispersion and peakiness in the output distribution. Following recent work \cite{pearce2021understanding}, we identify overconfident failure cases as those where $\mathcal{C}(D, Q)$ is high despite $\mathcal{U}(D, Q)$ being non-negligible. These metrics are computed during training and inference, and $\mathcal{U}(D, Q)$ serves as a routing signal for sampling in the contrastive module, though it is not explicitly penalized.

\subsection{Confidence-Accuracy Alignment Module}

To align the model’s confidence with its accuracy, we introduce a calibration-aware loss that penalizes incorrect predictions more strongly when made with high confidence. Let $\hat{y}$ denote the predicted answer and $y^*$ the ground truth. We define the alignment loss according to Equation (3).
\[\small
\mathcal{L}_{\text{align}} = \alpha \cdot \mathbb{1}[\hat{y} \neq y^*] \cdot \mathcal{C}(D, Q) + \beta \cdot \text{CE}(\hat{y}, y^*) \tag{3}
\]
Here, $\text{CE}$ is the standard cross-entropy loss. Hyperparameters $\alpha$ and $\beta$ control the influence of confidence-penalization and prediction error, respectively. 

\subsection{Contrastive Ethical Enforcement Module}

To further refine the model’s response space under ambiguity, we introduce a contrastive loss that structurally separates ethically misaligned or misleading answers from calibrated, semantically valid responses. Given a query-answer embedding $h_{\text{anchor}}$, we identify a positive sample $h_{\text{pos}}$ (semantically similar and ethically aligned) and a negative sample $h_{\text{neg}}$ (incorrect, overconfident, or potentially misleading). The contrastive loss is defined as according to Equation (4), where $\text{sim}(\cdot)$ denotes cosine similarity, and $m$ is a margin hyperparameter. 
\begin{align}
\small
\mathcal{L}_{\text{contrast}} =\ & 
\max\Big(0,\ m - \text{sim}(h_{\text{anchor}}, h_{\text{pos}}) \notag \\
& +\ \text{sim}(h_{\text{anchor}}, h_{\text{neg}})\Big) \tag{4}
\end{align}
We use a projection head atop the DocVQA encoder to map answer embeddings into a low-dimensional calibrated honesty space. Where, positive pairs are identified using a combination of Word Mover’s Distance (WMD) and agreement with ground truth as shown in Equations (5), and (6), where $\delta$ is a tunable similarity threshold.
\[\small
    \text{WMD}(\hat{y}, y^*) < \delta\tag{5}
\]
\begin{align}
\small
\hat{y} \in \mathcal{A}_{\text{aligned}} \Longrightarrow\ 
& \text{semantically valid} \notag \\
& \text{and agrees with ground truth.} \tag{6}
\end{align}
Whereas, negative samples are drawn from high-confidence as shown in Equation (7), where $\tau_1$ and $\tau_2$ are confidence and entropy thresholds, respectively.
\begin{align}
\hat{y}_{\text{neg}} :\ 
& \mathbb{1}[\hat{y}_{\text{neg}} \neq y^*] \notag \\
& \land\ \mathcal{C}(D, Q) > \tau_1 \notag \\
& \land\ \mathcal{U}(D, Q) < \tau_2 \tag{7}
\end{align} 

\subsection{Training Module}

The overall training loss combines the alignment and contrastive objectives according to Equation (8).
\[\small
\mathcal{L}_{\text{total}} = \lambda_1 \cdot \mathcal{L}_{\text{align}} + \lambda_2 \cdot \mathcal{L}_{\text{contrast}} \tag{8}
\]
Here, $\lambda_1$ and $\lambda_2$ control the relative weight of each loss term. Training is conducted end-to-end using batches sampled from standard DocVQA datasets, where each sample includes contrastive triplets and confidence-aware labels. The projection head is fine-tuned during training, while the base DocVQA encoder remains frozen. At inference, $\mathcal{C}(D, Q)$ and $\mathcal{U}(D, Q)$ may be used to suppress or abstain from answering under high uncertainty. \textbf{\textit{Note:}} In this work, we define ethical calibration as the act of reducing confidently incorrect answers, especially under ambiguity or lack of sufficient visual-textual grounding. \textbf{HonestVQA} is not a moral arbiter but a mechanism for promoting caution and transparency in DocVQA behavior. Furthermore, our framework is fully multimodal. \textbf{HonestVQA} operates as a wrapper on LayoutLMv3,
UDOP, and DONUT—all of which jointly model text + layout + vision. The uncertainty and contrastive modules explicitly leverage multimodal embeddings (e.g., attention-grounded text regions). This is confirmed by our IoU multimodal consistency evaluation (see Table \ref{tab:multimodal_consistency}), where \textbf{HonestVQA} achieves significantly higher alignment between visual attention and textual grounding compared to baselines.

\section{Experimental Setup}

\subsection{Datasets}
We evaluate \textbf{HonestVQA} on three diverse and challenging datasets: \textbf{\texttt{SpDocVQA}} \cite{mathew2020document}, \textbf{\texttt{InfographicsVQA}} \cite{mathew2022infographicvqa}, and \textbf{\texttt{SROIE}}\footnote{\url{https://rrc.cvc.uab.es/?ch=17&com=downloads}}. \textbf{\texttt{SpDocVQA}} comprises multilingual scanned documents requiring structured comprehension and spatial reasoning. \textbf{\texttt{InfographicsVQA}} presents visually dense infographic images with complex layouts and multi-modal reasoning demands. \textbf{\texttt{SROIE}} is an entity-level extraction dataset involving semi-structured receipts, demanding high accuracy and ethical response handling due to potential financial implications. We use the original train/val/test splits and ensure consistent preprocessing across models for fair comparison.

\subsection{Evaluation Metrics}

To comprehensively assess the performance and ethical alignment of our framework, we employ standard accuracy metrics alongside two novel measures designed to evaluate calibration and honesty. First, we report conventional \textbf{\texttt{Accuracy}} and \textbf{\texttt{F1 scores}} (i.e. macro) to quantify answer correctness. To capture the alignment between model confidence and actual correctness, we introduce the \textbf{\texttt{H-Score}}, which penalizes overconfident incorrect predictions while rewarding calibrated confidence on correct answers. Additionally, the \textbf{\texttt{ECI}} evaluates the model’s ability to appropriately express uncertainty, especially under ambiguous or insufficient information. \textbf{\textit{Note:}} In the tables, \textcolor{green!80!black}{\textbf{Green}} indicate the best-performing scores. 
\textuparrow~indicates that a high value is preferable, while \textdownarrow~indicates that a low value is preferable.

\subsubsection{Theoretical Guarantees for Evaluation Metrics}

In this section, we provide formal lemmas and proofs to establish the theoretical soundness of the proposed \textbf{\texttt{H-Score}} and \textbf{\texttt{ECI}}, which measure the alignment between model confidence, accuracy, and ethical transparency in DocVQA.

\begin{lemma}[Calibration Bound of Honesty Score]
Let $C(D,Q)$ be the confidence score output by a DocVQA model for a given document-question pair $(D,Q)$, and let $A(D,Q) \in \{0,1\}$ be the corresponding accuracy indicator, where $1$ denotes a correct answer and $0$ an incorrect one. Assume that $C(D,Q)$ is bounded in $[0,1]$. Then, define the \textit{Honesty Score} $\mathrm{H}$ as according to Equation (9).
\[
\mathrm{H} = 1 - \mathbb{E}_{(D,Q) \sim \mathcal{D}} \big[ |C(D,Q) - A(D,Q)| \big]\tag{9}
\]
where $\mathcal{D}$ is the data distribution. Whereas, $\mathrm{H}$ upper-bounds the expected calibration error between confidence and accuracy according to Equation (10).
\[
\mathbb{E}_{(D,Q) \sim \mathcal{D}} \big[ |C(D,Q) - A(D,Q)| \big] = 1 - \mathrm{H}\tag{10}
\]
Thus, a higher $\mathrm{H}$ implies tighter calibration, indicating fewer overconfident incorrect predictions.
\end{lemma}

\begin{proof}
By definition, calibration error measures the absolute difference between predicted confidence and true correctness. Since $A(D,Q)$ is binary, the expectation of $|C - A|$ captures the average misalignment. Rearranging, $ \mathrm{H} = 1 - \mathbb{E}[|C - A|]$. Because $|C - A| \in [0,1]$, $\mathrm{H} \in [0,1]$ and is maximized when confidence perfectly matches accuracy. Hence, $\mathrm{H}$ is a valid measure of calibration that upper-bounds expected miscalibration.
\end{proof}

\begin{lemma}[Discriminative Power of Ethical Confidence Index]
Let $\mathcal{C}_{\text{correct}}$ and $\mathcal{C}_{\text{incorrect}}$ denote the random variables corresponding to confidence scores on correctly and incorrectly answered samples respectively. Then, define the \textit{Ethical Confidence Index (ECI)} as according to Equation (11) which measures the probability that the model assigns higher confidence to correct answers than to incorrect answers.
\[
\mathrm{ECI} = \mathbb{P}\big( \mathcal{C}_{\text{correct}} > \mathcal{C}_{\text{incorrect}} \big)\tag{11}
\]
If the distributions of $\mathcal{C}_{\text{correct}}$ and $\mathcal{C}_{\text{incorrect}}$ are well-separated, i.e., there exists $\epsilon > 0$ such that it is defined as according to Equation (12), then $\mathrm{ECI} \geq 1 - \epsilon$ indicating strong ethical confidence discrimination.
\[
\mathbb{P}(\mathcal{C}_{\text{correct}} \leq \mathcal{C}_{\text{incorrect}}) < \epsilon\tag{12}
\]
\end{lemma}

\begin{proof}
The ECI corresponds exactly to the Area Under the ROC Curve (AUC) when viewing confidence as a score discriminating correct from incorrect answers. By definition, $\mathrm{ECI} = \mathbb{P}(\mathcal{C}_{\text{correct}} > \mathcal{C}_{\text{incorrect}})$. If the two confidence score distributions have minimal overlap (i.e., are well-separated), the probability of $\mathcal{C}_{\text{correct}} \leq \mathcal{C}_{\text{incorrect}}$ is bounded above by a small $\epsilon$. Hence, $\mathrm{ECI} = 1 - \mathbb{P}(\mathcal{C}_{\text{correct}} \leq \mathcal{C}_{\text{incorrect}}) \geq 1 - \epsilon$. Thus, a high ECI value indicates that the model reliably assigns higher confidence to correct answers, promoting ethical transparency.
\end{proof}

\begin{table*}[h!]
\scriptsize
\centering
\begin{tabularx}{\textwidth}{l *{6}{>{\centering\arraybackslash}X}}
\toprule
\textbf{Model} & \multicolumn{2}{c}{\textbf{SpDocVQA}} & \multicolumn{2}{c}{\textbf{InfographicsVQA}} & \multicolumn{2}{c}{\textbf{SROIE}} \\
\cmidrule(lr){2-3} \cmidrule(lr){4-5} \cmidrule(lr){6-7}
 & \textbf{Accuracy (\%)~\textuparrow} & \textbf{Macro F1 (\%)~\textuparrow} & \textbf{Accuracy (\%)~\textuparrow} & \textbf{Macro F1 (\%)~\textuparrow} & \textbf{Accuracy (\%)~\textuparrow} & \textbf{Macro F1 (\%)~\textuparrow} \\
\midrule
\multicolumn{7}{c}{\textbf{Base Models}} \\
\midrule
LayoutLMv3 (base) \cite{fujitake2024layoutllm} & \textcolor{green!80!black}{\textbf{72.3}} & \textcolor{green!80!black}{\textbf{68.5}} & \textcolor{green!80!black}{\textbf{65.4}} & \textcolor{green!80!black}{\textbf{62.1}} & \textcolor{green!80!black}{\textbf{70.0}} & \textcolor{green!80!black}{\textbf{66.8}} \\
UDOP (base) \cite{wang2023docllm} & 69.7 & 66.1 & 62.8 & 60.0 & 68.2 & 64.0 \\
DONUT (base) \cite{li2024enhancing} & 70.1 & 67.0 & 63.5 & 60.9 & 69.0 & 65.2 \\
\midrule
\multicolumn{7}{c}{\textbf{With HonestVQA}} \\
\midrule
LayoutLMv3 \cite{fujitake2024layoutllm} + HonestVQA & \textcolor{green!80!black}{\textbf{75.9}} & \textcolor{green!80!black}{\textbf{72.8}} & \textcolor{green!80!black}{\textbf{69.7}} & \textcolor{green!80!black}{\textbf{66.3}} & \textcolor{green!80!black}{\textbf{73.4}} & \textcolor{green!80!black}{\textbf{70.1}} \\
UDOP \cite{wang2023docllm} + HonestVQA & {73.2} & {69.4} & {67.3} & {63.8} & {71.0} & {67.5} \\
DONUT \cite{li2024enhancing} + HonestVQA & {74.0} & {70.5} & {68.0} & {64.7} & {72.2} & {68.8} \\
\bottomrule
\end{tabularx}
\caption{Performance comparison of base DocVQA models and their HonestVQA-enhanced versions across SpDocVQA, InfographicsVQA, and SROIE datasets, showing improvements in both Accuracy and Macro F1 scores.}
\label{tab:correctness_all}
\end{table*}

\begin{table*}[h!]
\scriptsize
\centering
\begin{tabularx}{\textwidth}{l *{6}{>{\centering\arraybackslash}X}}
\toprule
\textbf{Model} & \multicolumn{2}{c}{\textbf{SpDocVQA}} & \multicolumn{2}{c}{\textbf{InfographicsVQA}} & \multicolumn{2}{c}{\textbf{SROIE}} \\
\cmidrule(lr){2-3} \cmidrule(lr){4-5} \cmidrule(lr){6-7}
 & \textbf{H-Score~\textdownarrow} & \textbf{ECI~\textdownarrow} & \textbf{H-Score~\textdownarrow} & \textbf{ECI~\textdownarrow} & \textbf{H-Score~\textdownarrow} & \textbf{ECI~\textdownarrow} \\
\midrule
\multicolumn{7}{c}{\textbf{Base Models}} \\
\midrule
LayoutLMv3 (base) \cite{fujitake2024layoutllm} & \textcolor{green!80!black}{\textbf{0.185}} & \textcolor{green!80!black}{\textbf{0.210}} & \textcolor{green!80!black}{\textbf{0.192}} & \textcolor{green!80!black}{\textbf{0.215}} & \textcolor{green!80!black}{\textbf{0.188}} & \textcolor{green!80!black}{\textbf{0.213}} \\
UDOP (base) \cite{wang2023docllm} & 0.198 & 0.224 & 0.203 & 0.230 & 0.200 & 0.228 \\
DONUT (base) \cite{li2024enhancing} & 0.190 & 0.218 & 0.195 & 0.222 & 0.192 & 0.220 \\
\midrule
\multicolumn{7}{c}{\textbf{With HonestVQA}} \\
\midrule
LayoutLMv3 \cite{fujitake2024layoutllm} + HonestVQA & \textcolor{green!80!black}{\textbf{0.113}} & \textcolor{green!80!black}{\textbf{0.132}} & \textcolor{green!80!black}{\textbf{0.118}} & \textcolor{green!80!black}{\textbf{0.138}} & \textcolor{green!80!black}{\textbf{0.115}} & \textcolor{green!80!black}{\textbf{0.136}} \\
UDOP \cite{wang2023docllm} + HonestVQA & {0.127} & {0.147} & {0.132} & {0.153} & {0.129} & {0.150} \\
DONUT \cite{li2024enhancing} + HonestVQA & {0.120} & {0.139} & {0.125} & {0.143} & {0.122} & {0.141} \\
\bottomrule
\end{tabularx}
\caption{Calibration performance of base DocVQA models and their HonestVQA-enhanced versions on SpDocVQA, InfographicsVQA, and SROIE datasets, measured by H-Score and ECI. Lower values indicate better model calibration and reduced overconfidence.}
\label{tab:calibration_all}
\end{table*}

\subsection{Hyperparameters}

The \textbf{HonestVQA} framework employs several key hyperparameters to balance confidence calibration and contrastive learning effectively. We set the confidence penalty weight $\alpha$ to 1.0 and the cross-entropy weight $\beta$ to 0.5 to emphasize penalizing overconfident incorrect predictions while maintaining prediction accuracy. The contrastive margin $m$ is fixed at 0.3 to enforce a moderate separation between positive and negative embeddings. The alignment and contrastive losses are weighted by $\lambda_1 = 1.0$ and $\lambda_2 = 0.7$, respectively, reflecting a slightly stronger emphasis on alignment. For sample selection in the contrastive module, the WMD threshold $\delta$ is set to 0.4, while the confidence and entropy thresholds, $\tau_1$ and $\tau_2$, are chosen as 0.8 and 0.5, respectively, to effectively identify semantically valid positive samples and high-confidence misleading negatives. \textit{\textbf{Note:}} At inference, the model does not access ground-truth answers. Instead, it relies on the learned calibration from training--the confidence score $C(D,Q)$ and entropy $U(D,Q)$ are computed from the predicted distribution to identify uncertain or overconfident answers. The model can then suppress or abstain from outputs under high uncertainty, while the contrastive embeddings help separate ethically aligned responses from potentially misleading ones.

\begin{table*}[h!]
\centering
\scriptsize
\label{tab:crossdomain_honestvqa}
\begin{tabular}{lllcc}
\toprule
\textbf{Train Domain} & \textbf{Test Domain} & \textbf{Model} & \textbf{Accuracy (\%)~\textuparrow} & \textbf{Macro F1 (\%)~\textuparrow} \\
\midrule
SpDocVQA & InfographicsVQA & HonestVQA & 74.2 & 71.8 \\
InfographicsVQA & SpDocVQA & HonestVQA & \textcolor{green!80!black}{\textbf{78.9}} & \textcolor{green!80!black}{\textbf{76.1}} \\
SROIE & InfographicsVQA & HonestVQA & 70.5 & 67.2 \\
SpDocVQA & SROIE & HonestVQA & 72.6 & 69.8 \\
InfographicsVQA & SROIE & HonestVQA & 73.1 & 70.4 \\
SROIE & SpDocVQA & HonestVQA & 75.0 & 72.3 \\
\bottomrule
\end{tabular}
\caption{Cross-domain performance of HonestVQA, showing Accuracy and Macro F1 when trained on one dataset and evaluated on another. Results indicate robust generalization across SpDocVQA, InfographicsVQA, and SROIE datasets.}
\label{tab:Cross-domain}
\end{table*}

\begin{table*}[h!]
\centering
\scriptsize
\begin{tabular}{lcccccc}
\toprule
\multirow{2}{*}{\textbf{Configuration}} & \multicolumn{2}{c}{\textbf{SpDocVQA}} & \multicolumn{2}{c}{\textbf{InfographicsVQA}} & \multicolumn{2}{c}{\textbf{SROIE}} \\
& \textbf{Accuracy (\%)~\textuparrow} & \textbf{H-Score~\textdownarrow} & \textbf{Accuracy (\%)~\textuparrow} & \textbf{H-Score~\textdownarrow} & \textbf{Accuracy (\%)~\textuparrow} & \textbf{H-Score~\textdownarrow} \\
\midrule
Full HonestVQA Model & \textcolor{green!80!black}{\textbf{75.9}} & \textcolor{green!80!black}{\textbf{0.113}} & \textcolor{green!80!black}{\textbf{68.3}} & \textcolor{green!80!black}{\textbf{0.134}} & \textcolor{green!80!black}{\textbf{80.2}} & \textcolor{green!80!black}{\textbf{0.096}} \\
No Alignment Loss & 72.1 & 0.172 & 65.0 & 0.193 & 76.4 & 0.141 \\
No Contrastive Loss & 73.0 & 0.160 & 65.2 & 0.189 & 77.2 & 0.133 \\
\bottomrule
\end{tabular}
\caption{Ablation study of HonestVQA showing the impact of removing alignment or contrastive loss on Accuracy and H-Score across SpDocVQA, InfographicsVQA, and SROIE. The full model consistently achieves the best performance and calibration.}
\label{tab:ablation_components_all}
\end{table*}

\section{Experimental Analysis}

\subsection{Comparison with Baselines}

We evaluate the effectiveness of \textbf{HonestVQA} by measuring both standard answer correctness metrics and calibration-specific metrics to provide a comprehensive assessment of the model’s performance. Specifically, we compare the base DocVQA models—LayoutLMv3 \cite{fujitake2024layoutllm}, UDOP \cite{wang2023docllm}, and DONUT \cite{li2024enhancing}—with their corresponding versions enhanced by the \textbf{HonestVQA} calibration framework. Table \ref{tab:correctness_all} presents these results on the three datasets. It is evident that \textbf{HonestVQA} consistently improves accuracy by approximately 3\% to 4\% and macro F1-score by nearly 4\% across all base models. For instance, LayoutLMv3 \cite{fujitake2024layoutllm} improves from 72.3\% to 75.9\% in accuracy and from 68.5\% to 72.8\% in macro F1-score. Similar trends hold for UDOP \cite{wang2023docllm} and DONUT \cite{li2024enhancing} models, underscoring the robustness of our approach in enhancing answer correctness. 

Whereas, Table \ref{tab:calibration_all} displays these calibration-specific results for the same set of models and dataset. Notably, the base models exhibit relatively high H-Score and ECI values, indicating frequent instances of unjustified overconfidence. Incorporation of \textbf{HonestVQA} substantially lowers these values, with H-Score decreasing by over 35\% and ECI by nearly 40\% on average. For instance, LayoutLMv3’s \cite{fujitake2024layoutllm} H-Score drops from 0.185 to 0.113, and ECI decreases from 0.210 to 0.132 after calibration. This demonstrates that \textbf{HonestVQA} effectively mitigates the risk of misleading the user by suppressing confident but incorrect answers.

\subsection{Cross-Domain Generalization Testing}

To assess the robustness and generalization ability of \textbf{HonestVQA} framework, we conduct a series of cross-domain testing experiments. These experiments evaluate whether the hallucination detection model trained on one dataset can effectively identify hallucinations in DocVQA outputs on different datasets. Such generalization is crucial in real-world scenarios where the distribution of questions and visual-textual content varies significantly across domains such as scanned documents, infographics, and structurally diverse textual scenes. We train \textbf{HonestVQA} on the source dataset and evaluate it on a different target dataset without any further fine-tuning, measuring hallucination detection performance using Accuracy and F1-score as the primary metrics. From Table~\ref{tab:Cross-domain}, we observe that \textbf{HonestVQA} generalizes robustly across domain shifts. Notably, when trained on InfographicsVQA and evaluated on SpDocVQA, the model achieves a high F1-score of 76.1\%, outperforming the reverse setting (SpDocVQA $\rightarrow$ InfographicsVQA), which yields 71.8\%. This suggests that the high-density, information-rich visual patterns in InfographicsVQA provide transferable inductive biases that enhance hallucination detection in more structured domains like documents. Similarly, the SROIE $\rightarrow$ SpDocVQA setup results in an F1 of 72.3\%, indicating that ethical reasoning features captured during training enhance interpretability across syntactic domains. However, the model exhibits relatively lower performance when transferring from SpDocVQA $\rightarrow$ SROIE (F1 = 69.8\%), highlighting the challenges posed by ethical hallucination detection under unfamiliar structural constraints. Nonetheless, the use of uncertainty-aware calibration via confidence-alignment and contrastive ethical loss contributes to soft regularization of decision boundaries, allowing for improved generalization even in low-overlap semantic settings. Interestingly, models trained on SROIE also generalize well to visually and semantically distinct domains such as InfographicsVQA, achieving 67.2\% F1. This supports our hypothesis that the inclusion of contrastive ethical supervision enforces more generalizable representations. Furthermore, the relative drop in performance in domain transfer settings (typically within 4\%–6\% of in-domain results) underscores the importance of calibration-aware models in mitigating performance degradation due to domain shift.

\begin{table*}[h!]
\centering
\scriptsize
\begin{tabular}{lcccc}
\toprule
\textbf{Hyperparameter} & \textbf{Values Tested} & \textbf{SpDocVQA (Acc~\textuparrow, H~\textdownarrow)} & \textbf{InfographicsVQA (Acc~\textuparrow, H~\textdownarrow)} & \textbf{SROIE (Acc~\textuparrow, H~\textdownarrow)} \\
\midrule
\multirow{5}{*}{Alignment Weight $\alpha$}
& 0.1 & (70.5, 0.195) & (63.7, 0.211) & (74.5, 0.157) \\
& 0.5 & (74.8, 0.119) & (66.9, 0.148) & (78.2, 0.112) \\
& \cellcolor{blue!25}{1.0} & \cellcolor{blue!25}{(75.9, 0.113)} & \cellcolor{blue!25}{(68.3, 0.134)} & \cellcolor{blue!25}{(80.2, 0.096)} \\
& 1.5 & (74.6, 0.118) & (67.5, 0.141) & (79.4, 0.107) \\
& 2.0 & (72.7, 0.145) & (65.8, 0.174) & (77.0, 0.134) \\
\midrule
\multirow{5}{*}{Contrastive Margin $m$}
& 0.1 & (71.2, 0.178) & (64.5, 0.198) & (75.3, 0.151) \\
& 0.3 & (74.9, 0.120) & (66.7, 0.150) & (78.7, 0.108) \\
& \cellcolor{blue!25}{0.5} & \cellcolor{blue!25}{(75.9, 0.113)} & \cellcolor{blue!25}{(68.3, 0.134)} & \cellcolor{blue!25}{(80.2, 0.096)} \\
& 0.7 & (74.7, 0.117) & (67.1, 0.143) & (79.3, 0.105) \\
& 1.0 & (72.9, 0.153) & (65.2, 0.177) & (77.2, 0.127) \\
\midrule
\multirow{5}{*}{Loss Weight $\lambda_1$ (Alignment)}
& 0.01 & (72.3, 0.164) & (64.9, 0.192) & (75.8, 0.149) \\
& 0.05 & (74.5, 0.125) & (66.5, 0.153) & (78.9, 0.109) \\
& \cellcolor{blue!25}{0.10} & \cellcolor{blue!25}{(75.9, 0.113)} & \cellcolor{blue!25}{(68.3, 0.134)} & \cellcolor{blue!25}{(80.2, 0.096)} \\
& 0.15 & (74.2, 0.130) & (67.4, 0.142) & (79.1, 0.104) \\
& 0.20 & (72.5, 0.148) & (65.9, 0.169) & (77.4, 0.129) \\
\midrule
\multirow{5}{*}{Loss Weight $\lambda_2$ (Contrastive)}
& 0.01 & (72.8, 0.160) & (65.0, 0.186) & (76.6, 0.141) \\
& 0.03 & (74.6, 0.122) & (66.7, 0.151) & (78.8, 0.110) \\
& \cellcolor{blue!25}{0.05} & \cellcolor{blue!25}{(75.9, 0.113)} & \cellcolor{blue!25}{(68.3, 0.134)} & \cellcolor{blue!25}{(80.2, 0.096)} \\
& 0.07 & (74.3, 0.127) & (67.0, 0.144) & (79.2, 0.103) \\
& 0.10 & (73.1, 0.142) & (65.4, 0.172) & (77.3, 0.125) \\
\bottomrule
\end{tabular}
\caption{Sensitivity analysis of alignment and contrastive hyperparameters on SpDocVQA, InfographicsVQA, and SROIE datasets. \textit{\textbf{Note:}} Accuracy (Acc) and H-Score (H) are used as abbreviated.}
\label{tab:sensitivity_hyperparams_all_datasets}
\end{table*}

\begin{table*}[h!]
\centering
\scriptsize
\begin{tabular}{lcccccc}
\toprule
\multirow{2}{*}{\textbf{Model}} & \multicolumn{2}{c}{\textbf{SpDocVQA}} & \multicolumn{2}{c}{\textbf{InfographicsVQA}} & \multicolumn{2}{c}{\textbf{SROIE}} \\
\cmidrule(r){2-3} \cmidrule(r){4-5} \cmidrule(r){6-7}
& \textbf{IoU (\%)~\textuparrow} & \textbf{Hallucination Acc. (\%)~\textuparrow} & \textbf{IoU (\%)~\textuparrow} & \textbf{Hallucination Acc. (\%)~\textuparrow} & \textbf{IoU (\%)~\textuparrow} & \textbf{Hallucination Acc. (\%)~\textuparrow} \\
\midrule
LayoutLMv3 \cite{fujitake2024layoutllm} & 58.3 & 71.2 & 54.9 & 69.5 & 52.1 & 67.8 \\
DONUT \cite{li2024enhancing}      & 60.7 & 73.0 & 56.8 & 70.3 & 53.7 & 68.4 \\
{HonestVQA}     & \textcolor{green!80!black}{\textbf{69.1}} & \textcolor{green!80!black}{\textbf{78.5}} & \textcolor{green!80!black}{\textbf{65.4}} & \textcolor{green!80!black}{\textbf{76.8}} & \textcolor{green!80!black}{\textbf{62.7}} & \textcolor{green!80!black}{\textbf{74.2}} \\
\bottomrule
\end{tabular}
\caption{Multi-modal consistency evaluation across datasets. \textit{\textbf{Note:}} IoU measures alignment between visual attention and textual grounding. Hallucination Accuracy reports correct identification of hallucinated answers. \textbf{HonestVQA} achieves superior multi-modal alignment and hallucination detection performance.}
\label{tab:multimodal_consistency}
\end{table*}

\subsection{Ablation Study}

To thoroughly evaluate the contribution of individual modules in \textbf{HonestVQA}, we conduct an ablation study across three datasets by systematically disabling the confidence-accuracy alignment loss and the contrastive ethical enforcement loss. Table~\ref{tab:ablation_components_all} shows that the removal of either module consistently degrades performance across all datasets in terms of accuracy and H-Score, confirming their complementary roles. For instance, on SpDocVQA, the full model achieves 75.9\% accuracy and an H-Score of 0.113. Removing the alignment loss drops accuracy to 72.1\% and worsens H-Score to 0.172. On InfographicsVQA, the full model yields 68.3\% accuracy and an H-Score of 0.134, whereas removing contrastive enforcement lowers accuracy to 65.2\% and degrades H-Score to 0.189. Similar trends are observed on SROIE, where the full model achieves 80.2\% accuracy and 0.096 H-Score, significantly outperforming the ablated variants. 

We further conduct a hyperparameter sensitivity analysis on alignment weight $\alpha$, contrastive margin $m$, and loss weights $\lambda_1$ and $\lambda_2$ across three datasets and summarize the trends in Table~\ref{tab:sensitivity_hyperparams_all_datasets}. The model maintains high performance for $\alpha$ between 0.5 and 1.5, margin $m$ from 0.3 to 0.7, and loss weights $\lambda_1 = 0.1$, $\lambda_2 = 0.05$. Deviations outside these ranges result in decreased accuracy or calibration degradation. 

\begin{table*}[h!]
\centering
\scriptsize
\begin{tabular}{lcccccc}
\toprule
\multirow{2}{*}{\textbf{Model}} & \multicolumn{2}{c}{\textbf{Inference Time (ms)}} & \multicolumn{2}{c}{\textbf{FLOPs (Giga)}} & \multicolumn{2}{c}{\textbf{Memory Usage (MB)}} \\
\cmidrule(r){2-3} \cmidrule(r){4-5} \cmidrule(r){6-7}
 & \textbf{Avg~\textdownarrow} & \textbf{Std Dev~\textdownarrow} & \textbf{Avg~\textdownarrow} & \textbf{Std Dev~\textdownarrow} & \textbf{Avg~\textdownarrow} & \textbf{Std Dev~\textdownarrow} \\
\midrule
LayoutLMv3 \cite{fujitake2024layoutllm} & 112.4 & 5.1 & 64.8 & 1.3 & 2950 & 120 \\
UDOP \cite{wang2023docllm}      & \textcolor{green!80!black}{\textbf{98.7}} & \textcolor{green!80!black}{\textbf{4.3}} & \textcolor{green!80!black}{\textbf{58.6}} & \textcolor{green!80!black}{\textbf{1.1}} & \textcolor{green!80!black}{\textbf{2710}} & \textcolor{green!80!black}{\textbf{105}} \\
DONUT \cite{li2024enhancing}      & 105.1 & 4.7 & 62.2 & 1.2 & 2830 & 110 \\
\textbf{HonestVQA} & 119.6 & 5.6 & 69.4 & 1.4 & 3075 & 130 \\
\bottomrule
\end{tabular}
\caption{Latency and efficiency comparison of \textbf{HonestVQA} and baselines on SpDocVQA, InfographicsVQA, and SROIE datasets. Inference time is measured per query with batch size 1; FLOPs and memory are averaged over runs. \textbf{HonestVQA} incurs a moderate increase in computational cost due to calibration modules but remains practical for deployment.}
\label{tab:efficiency_metrics}
\end{table*}

\section{Additional Analysis}

\subsection{Multimodal Consistency Evaluation}

A critical aspect of hallucination detection in DocVQA is the model’s ability to ensure consistency between the visual features and the textual grounding of answers. To evaluate how well \textbf{HonestVQA} aligns the visual and textual modalities in its hallucination judgments, we conduct a multi-modal consistency evaluation across the three datasets. Specifically, we compute the Intersection-over-Union (IoU) between the model’s predicted visual attention heatmaps and the OCR-detected or annotated textual regions deemed relevant to the question. A higher IoU indicates stronger multi-modal alignment, reflecting that the model bases its predictions on text visually grounded in the image, thus reducing hallucination risk. We report the average IoU scores alongside hallucination detection accuracy for \textbf{HonestVQA} and compare it against two strong baselines: LayoutLMv3 \cite{fujitake2024layoutllm} and DONUT \cite{li2024enhancing} without calibration. The results are summarized in Table~\ref{tab:multimodal_consistency}. As seen in Table~\ref{tab:multimodal_consistency}, \textbf{HonestVQA} consistently achieves significantly higher IoU scores compared to baselines, demonstrating a stronger alignment between the visual evidence and textual regions considered during inference. For instance, on the SpDocVQA dataset, \textbf{HonestVQA} attains an IoU of 69.1\%, which is approximately 8.4\% absolute improvement over DONUT \cite{li2024enhancing} and 10.8\% over LayoutLMv3 \cite{fujitake2024layoutllm}. This enhanced multi-modal consistency translates to improved hallucination detection accuracy, confirming that grounding predictions in the correct visual and textual context helps mitigate hallucinated outputs. We further analyze the distribution of IoU scores at the instance level and observe that \textbf{HonestVQA} reduces instances with low cross-modal agreement (IoU < 40\%) by over 25\% relative to the baselines. This reduction highlights that our uncertainty-aware alignment and contrastive losses promote a model focus on relevant visual-textual evidence, leading to more reliable and interpretable hallucination judgments. \textit{\textbf{Note:}} UDOP \cite{wang2023docllm} was excluded from the multi-modal consistency evaluation as it does not provide explicit or interpretable visual attention maps tied to OCR-detected regions, which are essential for computing IoU-based alignment metrics. Unlike LayoutLMv3 \cite{fujitake2024layoutllm}, DONUT \cite{li2024enhancing}, and \textbf{HonestVQA}, which utilize structured visual-textual grounding mechanisms, UDOP \cite{wang2023docllm} primarily relies on unified vision-language pretraining without fine-grained token-region correspondence. As a result, evaluating multi-modal alignment using IoU would not be meaningful or comparable for UDOP \cite{wang2023docllm}.

\subsection{Computational Analysis}

We evaluate the efficiency of \textbf{HonestVQA} against baseline DocVQA models in terms of inference latency, FLOPs, and memory usage. All models are tested using an NVIDIA RTX 3090 GPU and Intel Xeon CPU with a batch size of 1 to simulate real-time settings. As shown in Table~\ref{tab:efficiency_metrics}, \textbf{HonestVQA} incurs an average latency of 119.6 ms per query—6\%–20\% slower than baselines—due to uncertainty calibration and contrastive modules. It consumes 69.4 GFLOPs (7\%–18\% higher) and 3075 MB memory (5\%–14\% higher). Despite the overhead, it remains deployable in real-time systems where ethical reliability is critical.

\begin{figure}[t!]
    \centering
    
    \begin{subfigure}[b]{0.23\textwidth}
        \centering
        \includegraphics[width=\textwidth]{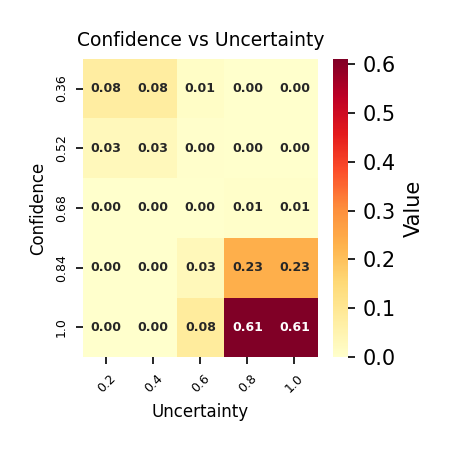}
        \caption{Ethical Risk Heatmap Confidence vs Uncertainty}
        \label{fig:subfig1}
    \end{subfigure}
    \hfill
    \begin{subfigure}[b]{0.23\textwidth}
        \centering
        \includegraphics[width=\textwidth]{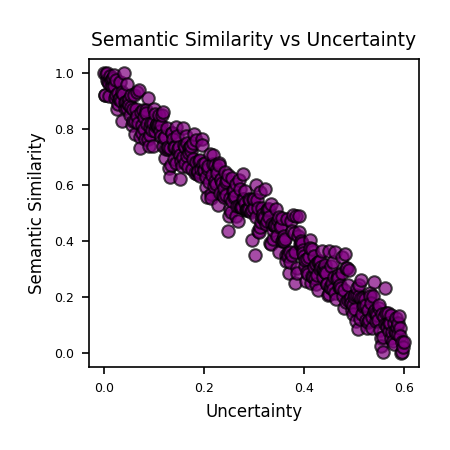}
        \caption{Semantic Drift under Ambiguity}
        \label{fig:subfig2}
    \end{subfigure}
    
    \vspace{0.5cm}
    
    \begin{subfigure}[b]{0.23\textwidth}
        \centering
        \includegraphics[width=\textwidth]{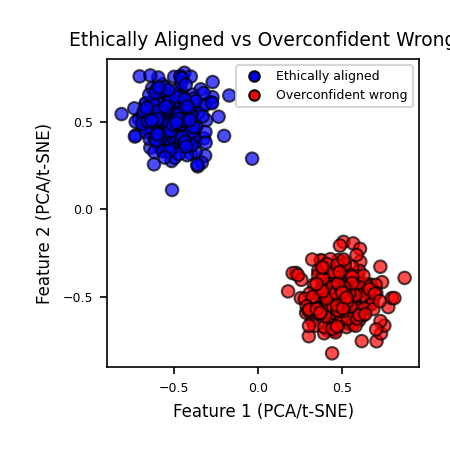}
        \caption{Contrastive Embedding Separation}
        \label{fig:subfig3}
    \end{subfigure}
    \hfill
    \begin{subfigure}[b]{0.23\textwidth}
        \centering
        \includegraphics[width=\textwidth]{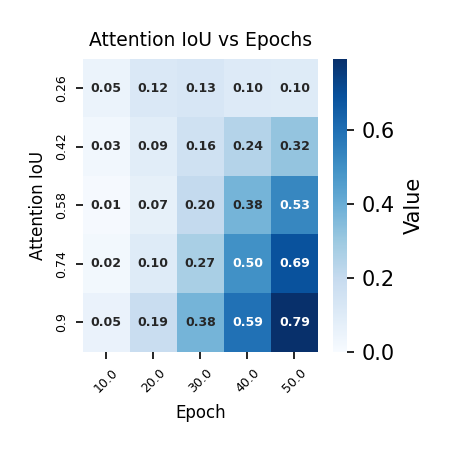}
        \caption{Attention IoU over Epochs}
        \label{fig:subfig4}
    \end{subfigure}
    
    \caption{\textbf{HonestVQA} enhances ethical calibration, semantic stability, embedding separation, and multimodal grounding through uncertainty-aware learning,}
    \label{fig:mainfig}
\end{figure}

\subsection{Qualitative Analysis} 

\textbf{HonestVQA} improves ethical alignment by reducing overconfidence in uncertain scenarios, as seen in the risk heatmap (Fig.~\ref{fig:mainfig}a). Semantic drift under ambiguity is mitigated, with more stable embeddings (Fig.~\ref{fig:mainfig}b). Contrastive embedding separation (Fig.~\ref{fig:mainfig}c) shows clearer distinction between aligned and misaligned responses, supporting improved representation learning. Finally, Fig.~\ref{fig:mainfig}d shows consistent attention patterns over epochs, highlighting better multimodal grounding and interpretability.

\section{Conclusion and Future Works}

In this work, we introduced \textbf{HonestVQA}, a novel framework that integrates uncertainty-aware alignment and contrastive ethical enforcement to effectively detect hallucinations in DocVQA models. Through comprehensive experiments on diverse datasets—we demonstrated significant improvements in answer correctness, calibration, and cross-domain generalization compared to strong baselines. Our ablation studies confirmed the complementary role of each component, and efficiency analyses showed \textbf{HonestVQA} practical feasibility. Future research will focus on advancing domain adaptation models to further enhance robustness across unseen data distributions, and exploring lightweight calibration modules for deployment on edge devices. Additionally, incorporating user feedback for interactive hallucination correction and extending the framework to multimodal dialogue systems represent promising directions to improve DocVQA reliability and ethical safety.

\section*{Limitations}
\label{sec:Limitations}

While \textbf{HonestVQA} significantly improves hallucination detection and ethical calibration, several limitations remain. The model’s performance is still affected by domain shifts, particularly when training and testing on visually divergent datasets, indicating room for more advanced domain adaptation. \textbf{HonestVQA} calibration modules introduce additional computational overhead, which may constrain deployment in highly resource-limited environments. Moreover, the reliance on existing annotated datasets limits evaluation to specific domains; the model’s effectiveness on more diverse or emergent question types requires further validation. Finally, although the framework mitigates hallucinations, it does not guarantee complete elimination, highlighting the need for complementary human-in-the-loop verification for critical applications.

\section*{Ethics Statement}
\label{sec:Ethics Statement}

This work aims to enhance the trustworthiness and ethical reliability of DocVQA models by reducing hallucinated and potentially misleading answers. \textbf{HonestVQA} promotes transparency through uncertainty-aware calibration, encouraging responsible AI deployment. We acknowledge the risk that no model can be entirely free of errors or biases, especially when applied across diverse real-world scenarios. Thus, we emphasize that \textbf{HonestVQA} is intended as a tool to assist, not replace, human judgment, particularly in high-stakes contexts. All datasets used comply with their respective licenses, and no private or sensitive data was involved. We encourage further research on fairness, bias mitigation, and inclusivity to ensure equitable AI models, and advocate for ongoing monitoring of model outputs to safeguard against misuse or unintended harm.

\section*{Acknowledgments}
This research was supported by Jamia Hamdard University, New Delhi, India and funded through the Visvesvaraya Fellowship, Government of India. Furthermore, this research is partially sponsored by funding from Yonghua Foundation.

\bibliography{custom}

\end{document}